\documentclass{article}
\usepackage[T1]{fontenc}
\usepackage{iclr2027_conference,times}
\usepackage{amsmath,amssymb,amsthm}
\usepackage{graphicx}
\usepackage{booktabs,array,tabularx}
\usepackage{microtype}
\usepackage{flafter}
\usepackage[section]{placeins}
\usepackage{needspace}
\usepackage{hyperref}
\usepackage{url}
\hypersetup{
  colorlinks=false,
  pdfborder={0 0 1},
  citebordercolor={0 1 0},
  linkbordercolor={1 1 1},
  urlbordercolor={1 1 1},
  filebordercolor={1 1 1}
}
\usepackage{enumitem}
\setcitestyle{authoryear,round,citesep={;},aysep={,},yysep={,}}

\iclrfinalcopy
\newtheorem{definition}{Definition}
\newtheorem{theorem}{Theorem}[section]
\newtheorem{lemma}[theorem]{Lemma}
\newtheorem{proposition}[theorem]{Proposition}
\newtheorem{corollary}[theorem]{Corollary}
\newcommand{\Demand}{\operatorname{Demand}}
\newcommand{\Def}{\operatorname{Def}}
\newcommand{\GS}{\operatorname{GapScore}}
\newcommand{\Col}{\operatorname{Collision}}
\newcommand{\Novelty}{\operatorname{Novelty}}
\newcommand{\PAWC}{\operatorname{PAWC}}
\newcommand{\VS}{\operatorname{VS}}
\newcommand{\R}{\mathbb{R}}
\newcommand{\sphere}{\mathbb{S}^{d-1}}
\newcommand{\norm}[1]{\left\lVert#1\right\rVert}
\newcommand{\inner}[2]{\left\langle#1,#2\right\rangle}
\setlist[itemize]{noitemsep,leftmargin=*,topsep=0pt}

\title{Gap Detection: Mitigating Homogenization\\through Demand and Defensibility}
\author{
Xucheng Yu \\
University of Illinois Urbana-Champaign \\
\texttt{xy63@illinois.edu} \And
Haibo Jin \\
University of Illinois Urbana-Champaign \\
\texttt{haibo@illinois.edu} \And
Xiaoqi Wang \\
Starc Institute \\
\texttt{maiichann9106@gmail.com} \AND
Runze Kong \\
University of Michigan, Ann Arbor \\
\texttt{runzekog22@outlook.com} \And
Haohan Wang \\
University of Illinois Urbana-Champaign \\
\texttt{haohanw@illinois.edu}
}

\begin{document}
\maketitle
\lhead{}

\begin{abstract}
Language models often produce semantically similar answers to open-ended questions, even when many plausible alternatives exist. Generating useful alternatives requires identifying both what the query calls for and what existing responses already cover. We introduce Gap Detection, a framework that searches for relevant but unoccupied directions in a shared text embedding space. Its objective, GapScore, combines query relevance (Demand) with dissimilarity to the nearest existing item (Defensibility). The framework optimizes this product on the unit sphere, maps candidate directions to concept words with a language model, and uses task-specific filtering and constrained prompting to turn those concepts into text. On five curated topics and 90 INFINITY-CHAT queries, the reported mean relative reductions in output Collision are 12.8\% and 10.8\%, respectively. We also instantiate the framework for generative engine optimization over four product markets and for iterative synthetic-data augmentation. In the latter setting, with 100 training examples per method, Gap Detection reduces Collision from 0.324 to 0.298 relative to undirected accumulation and increases the Vendi Score from 27.16 to 29.20; the accuracy difference is one example on a 40-example test set. Component ablations expose the trade-off between semantic separation and relevance, while also showing that better direction scores do not uniformly yield better downstream metrics. These results support explicit gap search as a practical mechanism for diversifying content, with task validity checked separately from geometric novelty.
\end{abstract}

\section{Introduction}
\label{sec:intro}
Open-ended questions invite multiple reasonable answers, yet language models frequently return variations of the same idea. A request for a metaphor about time, for example, can repeatedly elicit familiar images such as rivers or weaving. This pattern has been documented within and across models using INFINITY-CHAT, a collection of real-world open-ended queries \citep{jiang2025hivemind}. Reference-free diversity metrics such as the Vendi Score likewise quantify the small effective number of semantic modes in such sets \citep{friedman2023vendi}. For a user seeking alternative perspectives, lexical variation is insufficient: answers must differ in substance while continuing to address the question.

This problem suggests a distinction between sampling another answer and identifying an answer that is missing. Changing the decoding distribution can alter which tokens are sampled, but it does not explicitly represent the semantic content already covered by a collection of responses. Existing diversity-aware methods such as Maximal Marginal Relevance (MMR) trade query relevance against redundancy when reranking a fixed candidate set \citep{carbonell1998mmr}; learning-to-rank methods likewise score those candidates \citep{liu2009learning}. Neither formulation searches unoccupied semantic space. Conversely, maximizing distance from previous responses alone can produce irrelevant concepts. A useful method must therefore keep new content connected to the user's demand while avoiding the nearest existing answer.

To address this problem, we introduce \textbf{Gap Detection}, a framework for discovering and realizing unoccupied semantic directions. Given a query and an occupied content set, the framework embeds both in a shared vector space \citep{mikolov2013word2vec}, then searches for directions with high \emph{GapScore}: the product of \emph{Demand} (query relevance) and \emph{Defensibility} (separation from the nearest occupied item). Directions are selected sequentially and added to the occupied set, so later searches account explicitly for earlier coverage.

A high-scoring vector, however, is not yet a usable answer. Gap Detection bridges this gap through language-based realization. Embedding retrieval supplies a short list of words near a discovered direction; Qwen selects concepts that can answer the query, and a second filter checks their contextual plausibility. A text generator then uses these concepts while being prompted to avoid frequent concepts and sentence patterns in existing responses. Finally, the generated text is evaluated for novelty and repetition. The geometric search identifies where to explore, while the mapping and validation stages determine whether the proposed direction can be expressed as meaningful content.

The same relevance--coverage distinction appears beyond open-ended answering. In generative engine optimization (GEO), a candidate position should address user needs while differing from competitors' claims; GEO work shows that supplied content can affect which sources a generative engine cites \citep{aggarwal2024geo}. Recursive training on generated data can reduce diversity \citep{shumailov2024collapse,alemohammad2024mad}, whereas retaining real and synthetic data can alter that trajectory \citep{gerstgrasser2024accumulating}. We study these as separate instantiations of the same objective, with setting-specific inputs, validity checks, and evaluation criteria.

Experiments examine both the effectiveness and the limits of this design. The reported mean relative Collision reduction is 12.8\% over five curated topics and 10.8\% over 90 open-ended queries. Additional experiments vary the source of the occupied responses and the model used to realize new directions. In four product markets, discovered positions receive citations in a controlled generative-engine simulation. In a small sentiment-data experiment, directed augmentation improves diversity relative to an accumulation baseline at the same final sample count. Ablations show that the product objective improves its intended relevance--separation score, but the full pipeline does not dominate every output-level metric.

Our contributions are threefold:
\begin{itemize}
\item \textbf{A coverage-aware objective.} We formulate gap discovery as joint optimization of Demand and Defensibility relative to an explicit occupied set.
\item \textbf{A search-to-text pipeline.} We combine spherical optimization, language-model concept mapping, and task-specific validation so that a geometric direction is tested before it is accepted as content.
\item \textbf{A cross-setting evaluation.} We evaluate the pipeline across open-ended generation, GEO positioning, and synthetic-data augmentation, reporting component trade-offs alongside aggregate gains. Appendix~\ref{app:theory} provides the geometric existence result, conditional optimization results, and their proofs; the main text focuses on how the method operates and what the experiments establish.
\end{itemize}

\section{Related Work}
\label{sec:related}
\paragraph{Output diversity and redundancy.}
Prior work shows that semantically repetitive outputs persist across model families and decoding configurations \citep{jiang2025hivemind}. These findings motivate evaluating similarity at the level of complete responses. We use mean pairwise embedding similarity, termed Collision here, and complement it with the Vendi Score, a spectral measure of diversity \citep{friedman2023vendi}. Gap Detection uses the occupied responses as an input to generation, rather than only as an evaluation set.

\paragraph{Relevance-aware diversification.}
Maximal Marginal Relevance (MMR) balances query relevance against similarity to previously selected documents \citep{carbonell1998mmr}. Gap Detection shares this redundancy-aware perspective. Its differences are a multiplicative objective, continuous search over embedding directions, and a realization stage that generates new text rather than only reordering existing documents. The product avoids an additive trade-off coefficient, but the complete pipeline still contains optimization and filtering hyperparameters. Prompted avoidance of repeated words is related to lexical constraints \citep{hokamp2017lexically}; unlike constrained decoding, our prompts do not guarantee that the generator obeys every constraint.

\paragraph{GEO and synthetic-data feedback.}
GEO studies how source content influences its visibility in generated answers \citep{aggarwal2024geo}. We consider an earlier decision: which plausible position should new content express relative to existing claims? Separately, recursive synthetic-data training can degrade generative distributions \citep{shumailov2024collapse,alemohammad2024mad}, while retaining real and past synthetic data can mitigate collapse in studied settings \citep{gerstgrasser2024accumulating}. Our augmentation experiment asks whether directing new samples toward semantic gaps adds value beyond retaining more samples. The documented procedure evaluates generated text with a downstream classifier and does not specify generative-model retraining, so it tests a narrower setting.

\section{Gap Detection}
\label{sec:method}
\subsection{Overview and Problem Setup}
\label{sec:overview}
Gap Detection takes an occupied content set $\mathcal C=\{c_1,\ldots,c_N\}$ and a demand set $\mathcal Q=\{(q_j,f_j)\}_{j=1}^M$, where $f_j\geq0$ and $\sum_j f_j=1$. For open-ended answering, $\mathcal C$ contains existing responses and $\mathcal Q$ can contain a single query with weight one. Let $E(x)\in\sphere$ be a normalized text embedding. We write $\mathbf e_i=E(c_i)$ and $\mathbf u_j=E(q_j)$, so cosine similarity is the inner product.

The output is a collection of proposed directions and their textual realizations. The procedure has four stages: represent demand and occupied content; search for directions that balance relevance and separation; map the directions to plausible concepts; and generate and check text. These stages address different failure modes. Search discourages repetition, concept mapping makes vectors interpretable, and validation rejects directions or outputs that do not satisfy the task. We describe the shared mechanism first and then specify its domain adaptations.

\begin{figure}[t]
\centering
\includegraphics[width=0.88\linewidth]{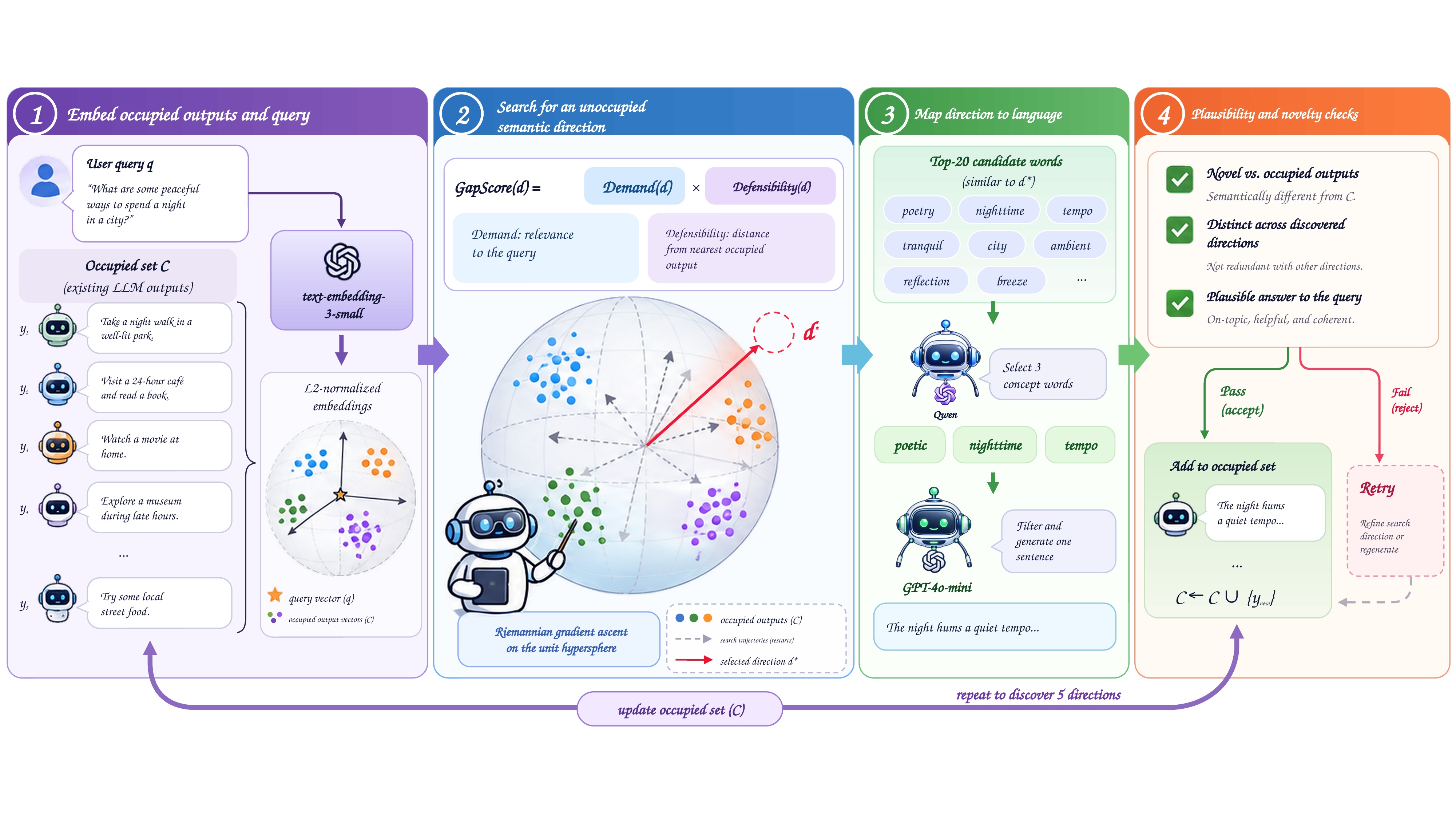}
\caption{Overview of Gap Detection. The pipeline embeds existing content and the query, searches the unit hypersphere for a high-GapScore direction, maps that direction to a concrete concept, and verifies the resulting candidate before accepting it. A rejected candidate triggers another search; an accepted direction is added to the occupied set so later directions cover a different region.}
\label{fig:pipeline}
\end{figure}

Figure~\ref{fig:pipeline} summarizes the search-to-text interface. The occupied set and query enter the same embedding space, while only the direction is passed to the concept stage through retrieved candidate words. The verifier closes the loop: it checks coherence, task relevance, and novelty, then either returns an accepted output or feeds the rejection back into the next direction search. This separation makes it possible to diagnose whether a failure comes from geometric search, concept mapping, or realization.

\subsection{Scoring Relevant and Unoccupied Directions}
\label{sec:objective}
\begin{definition}[Demand]
For a unit direction $\mathbf z$, the frequency-weighted relevance to the demand set is
\begin{equation}
\Demand(\mathbf z)=\sum_{j=1}^M f_j\inner{\mathbf u_j}{\mathbf z}
=\inner{\mathbf v}{\mathbf z},\qquad
\mathbf v=\sum_{j=1}^M f_j\mathbf u_j.
\label{eq:demand}
\end{equation}
\end{definition}
Demand preserves the connection to the task. For one query it is simply query--direction cosine similarity; for multiple queries it weights their contributions. These weights represent the supplied demand specification, not a learned relevance model.

\begin{definition}[Defensibility]
For an occupied embedding set $\mathcal O$, define
\begin{equation}
\Def(\mathbf z;\mathcal O)=1-\max_{\mathbf e\in\mathcal O}\inner{\mathbf e}{\mathbf z}.
\label{eq:def}
\end{equation}
\end{definition}
The nearest occupied item determines this score. A candidate close to even one existing item is penalized, regardless of its distance from the others. Defensibility lies in $[0,2]$ for unit vectors; it can exceed one when every occupied item has negative similarity to the direction. It measures geometric separation, not factual validity or commercial defensibility.

\begin{definition}[GapScore]
The joint direction score is
\begin{equation}
\GS(\mathbf z;\mathcal O)=\Demand(\mathbf z)\Def(\mathbf z;\mathcal O).
\label{eq:gap}
\end{equation}
\end{definition}
For nonnegative Demand, the product rewards both relevance and separation and vanishes when either is zero. Directions with negative Demand receive nonpositive scores. Thus, an empty region is not automatically desirable: it must also align with demand. The product can still trade off two positive factors, so it supplies a preference rather than a hard lower bound on either factor. Appendix~\ref{app:product} discusses its connection to bargaining objectives under explicit assumptions.

Demand-only optimization ignores the occupied set, whereas Defensibility-only optimization can leave the query. Their product couples relevance and separation; concept and output checks then supply the task-specific validity test.

\subsection{Searching for Multiple Directions}
\label{sec:search}
We approximately maximize Eq.~\eqref{eq:gap} on the unit sphere. Let $\mathbf e_*$ be a nearest occupied vector at the current iterate. Where that nearest vector is unique, the Euclidean gradient is
\begin{equation}
\mathbf a=\Def(\mathbf z;\mathcal O)\mathbf v-\Demand(\mathbf z)\mathbf e_*.
\label{eq:gradient}
\end{equation}
The first term moves toward demand, and the second moves away from the closest occupied direction. We project this update onto the sphere's tangent space and normalize after stepping:
\begin{equation}
\mathbf g=(I-\mathbf z\mathbf z^\top)\mathbf a,\qquad
\mathbf z\leftarrow\frac{\mathbf z+\eta\mathbf g}{\norm{\mathbf z+\eta\mathbf g}}.
\label{eq:update}
\end{equation}
This is a tangent update followed by a normalization retraction \citep{absil2008optimization}. At a nearest-neighbor tie, an active neighbor supplies the update direction. Multiple starts reduce reliance on any one initialization; they do not certify global optimality for this nonsmooth objective.

To obtain $K$ directions, the search at round $k$ uses
\begin{equation}
\mathcal O_k=\{\mathbf e_1,\ldots,\mathbf e_N\}\cup
\{\mathbf z_1,\ldots,\mathbf z_{k-1}\}.
\label{eq:occupied}
\end{equation}
Adding a selected direction makes exact rediscovery have zero Defensibility at the next round. The procedure therefore encourages successive directions to explore different regions while retaining the same demand vector. The reported implementation uses demand-aligned, perturbed-demand, and random starts, 15 restarts per strategy, 100 updates per restart, and $\eta=0.05$. The answering and GEO experiments use $K=5$.

The geometry also provides a useful diagnostic. If demand has a nonzero component orthogonal to the span of occupied content, normalizing that component gives a direction with Defensibility exactly one and positive Demand. This yields a lower bound on the best achievable geometric score. It does not guarantee that a corresponding natural-language answer exists. The precise condition, proof, and numerical comparison are in Appendix~\ref{app:geometry}; the language stages below remain necessary even when the bound holds.

\subsection{Mapping Directions to Plausible Concepts}
\label{sec:mapping}
A direction must be translated into concepts before a generator can use it. We construct a vocabulary from the Qwen2.5-3B-Instruct tokenizer, filter it to 13,409 English words, and embed these words with the same encoder used for the content. A query- and corpus-dependent screening score retains 10,000 words. For each direction, cosine retrieval then selects 20 unused candidate words near that direction. This screening is computational; it does not decide whether a word answers the query.

Qwen receives the query and the candidate words and selects up to three concepts that could form concrete answers. GPT-4o-mini subsequently checks whether the selected concepts fit the requested kind of answer, including whether a proposed word merely restates the query. If fewer than three words pass, the pipeline examines further candidates, with at most ten rounds. Accepted words are not reused across directions. Appendix~\ref{app:implementation} gives the vocabulary score and candidate-handling details, and Appendix~\ref{app:prompts} gives the available prompt templates.

This separation matters because geometric proximity and linguistic usefulness are different tests. A nearby token may be a synonym of the question rather than an answer to it; a highly separated token may be unrelated to the task. The language-model stages evaluate these distinctions after the embedding search. They provide a plausibility screen, not a correctness certificate.

\subsection{Generating and Checking Text}
\label{sec:realization}
For each accepted concept group, GPT-4o-mini generates one sentence conditioned on the original query. The prompt also includes two kinds of avoidance instructions derived from the occupied content. First, the pipeline identifies 12 frequent words close to the occupied-content centroid and asks the generator to avoid them. Second, it extracts recurring sentence patterns and asks for a different construction. Concept guidance specifies what to express, while these instructions discourage a return to the content and phrasing already represented in $\mathcal C$.

The realized sentences are embedded again. A novelty check requires the maximum similarity to the original occupied responses to be below 0.7. A separate check compares opening words, and another checks whether the maximum pairwise similarity among the new sentences is below 0.8. The opening-word check is lexical; the two similarity checks are semantic proxies. Neither establishes human-perceived creativity or factual accuracy. Measuring the realized text is necessary because a generator may move away from the direction suggested by its concept words.

The concept words form an inspectable intermediate representation: a failure can be attributed to direction search, concept selection, or sentence realization.

\subsection{Adapting the Framework to GEO and Synthetic Data}
\label{sec:adapt}
\paragraph{GEO positioning.}
The occupied set becomes a collection of competitor product claims, and demand consists of weighted search queries. The concept filter asks whether a candidate is plausible within the product category; the generator turns accepted words into a proposed claim. A controlled citation simulation then measures its visibility. Candidate claims are ranked by direction GapScore multiplied by the visibility score. Because the framework does not receive the target product's specifications, its output is a proposed market position requiring factual verification before use.

\paragraph{Synthetic-data augmentation.}
The occupied set becomes the current synthetic sample pool, and task-class descriptions with class-frequency weights specify demand. The concept filter checks task relevance, and the generator creates labeled examples. For sentiment classification, VADER \citep{hutto2014vader} supplies an additional polarity check: compound scores above 0.05 pass as positive and below $-0.05$ as negative. Candidate words are also reranked by absolute polarity. Directed methods retry generation up to four rounds to reach the accepted-example target. Final comparisons use $K=1$ per label and generation; a separate direction-set analysis uses $K=3$. The downstream classifier is trained on the resulting data and evaluated on a disjoint real test set. Its accuracy is an outcome measure, not an input to direction selection.

\section{Experiments}
\label{sec:experiments}
We ask whether Gap Detection reduces repetition in realized outputs, how each score component affects the result, and whether the search mechanism remains useful when the occupied content changes from responses to claims or training examples. The three settings share the embedding and search machinery but use separate task-specific evaluation protocols.

\subsection{Evaluation Protocol and Metrics}
\label{sec:metrics}
All text embeddings use \texttt{text-embedding-3-small} with 1,536 dimensions and L2 normalization. Qwen2.5-3B-Instruct-Q6\_K supplies concept decoding; GPT-4o-mini supplies the default plausibility filtering and text generation. Unless explicitly varied, these components are fixed across the direction-selection ablations.

\begin{definition}[Collision and novelty]
For an embedded text set $X=\{x_i\}_{i=1}^n$, $n\geq2$, and a generated text $g$, define
\begin{align}
\Col(X)&=\frac{2}{n(n-1)}\sum_{i<j}\inner{E(x_i)}{E(x_j)},\label{eq:collision}\\
\Novelty(g;\mathcal C)&=1-\max_{c\in\mathcal C}\inner{E(g)}{E(c)}.
\label{eq:novelty}
\end{align}
\end{definition}
Lower Collision indicates lower mean similarity; a value of zero means zero average similarity, not necessarily pairwise orthogonality. Novelty measures separation of a realized text from its nearest original reference. We report relative Collision reduction as $100(C_{\rm before}-C_{\rm after})/C_{\rm before}$ and preserve the reported mean of these relative reductions. A mean relative reduction need not equal the reduction computed from two aggregate means. The 90-query novelty statistic is the fraction of queries whose best generated sentence has Novelty greater than 0.3.

For synthetic datasets, we also use the Vendi Score \citep{friedman2023vendi}, $\VS=\exp(-\sum_i\lambda_i\log\lambda_i)$, where $\lambda_i$ are eigenvalues of the normalized cosine Gram matrix. It captures spectral diversity beyond a mean similarity. Downstream utility is measured by classification accuracy. These metrics are distinct from direction GapScore: optimizing an embedding objective is not itself evidence of improved output quality. The available results are point estimates; no repeated-run uncertainty estimates are supplied.

\subsection{Open-Ended Generation}
\label{sec:llm_exp}
\paragraph{Setup.}
The topic study uses five manually selected themes, with five query variants per theme and 20 GPT-4o responses as the occupied set. A second study uses 90 manually selected INFINITY-CHAT queries. Five directions and corresponding sentences are produced for each query. Table~\ref{tab:llm_main} summarizes the before/after Collision measurements. Detailed query-group results and reporting qualifications are given in Appendix~\ref{app:llm_results}.

\begin{table}[htbp]
\caption{Reported Collision before and after Gap Detection. The last column is the reported mean relative reduction. Topic averages and the 90-query aggregate are separate evaluations.}
\label{tab:llm_main}
\centering\small
\begin{tabular}{@{}lrrr@{}}
\toprule
Evaluation & Before $\downarrow$ & After $\downarrow$ & Reduction $\uparrow$\\
\midrule
Time metaphor & 0.7362 & 0.6573 & 10.7\%\\
Loneliness metaphor & 0.7449 & 0.6634 & 10.9\%\\
Redefine courage & 0.8046 & 0.7160 & 11.0\%\\
Memory as substance & 0.8163 & 0.7080 & 13.3\%\\
Hope via natural elements & 0.8178 & 0.6712 & 17.9\%\\
\midrule
Five-topic mean (reported) & 0.784 & 0.683 & 12.8\%\\
90 INFINITY-CHAT queries & 0.7766 & 0.6915 & 10.8\%\\
\bottomrule
\end{tabular}
\end{table}

\paragraph{Results.}
All five topics exhibit lower reported Collision, with reductions from 10.7\% to 17.9\%. The 90-query experiment shows a similar aggregate pattern, with Collision decreasing from 0.7766 to 0.6915 and a largest single-query reduction of 20.5\%. The best sentence passes the novelty threshold on 88 of 90 queries (97.8\%). These observations indicate separation from existing responses under the chosen encoder; they do not establish that the sentences are better answers according to human judgments.

The study does not include a matched rerun of temperature or min-$p$ baselines. The decoding experiments in prior work \citep{jiang2025hivemind} provide motivation, but their results are not a controlled numerical baseline for this pipeline. Likewise, direction search, concept filtering, and avoidance prompting contribute jointly to the observed end-to-end change. The ablation below isolates the direction-selection rule within that shared pipeline.

\paragraph{Component ablation.}
We compare random directions, Demand-only search, Defensibility-only search, and the product objective on five topics (Table~\ref{tab:llm_ablation}). GapScore achieves the highest direction score, 0.375, compared with 0.298 for Demand-only and $-0.228$ for Defensibility-only. However, Defensibility-only achieves the largest Collision reduction, and random directions achieve the highest novelty-check pass rate. The full method's pass rate is 84\% in this separate ablation. Thus, both terms help balance the geometric objective, but the ablation does not establish uniform superiority in realized-text metrics. The relatively strong random-direction result also suggests that the shared filtering and prompting stages make a substantial contribution.

\begin{table}[htbp]
\caption{Direction-selection ablation on five topics. Direction GapScore evaluates the optimized vectors; the other columns evaluate realized outputs. Best values in each column are bold. These runs are separate from Table~\ref{tab:llm_main}.}
\label{tab:llm_ablation}
\centering\small
\begin{tabular}{@{}lrrr@{}}
\toprule
Method & Direction GapScore $\uparrow$ & Collision reduction $\uparrow$ & Novelty pass $\uparrow$\\
\midrule
Random direction & 0.003 & 10.3\% & \textbf{100\%}\\
Demand only & 0.298 & 7.8\% & 92\%\\
Defensibility only & $-0.228$ & \textbf{11.5\%} & 96\%\\
GapScore & \textbf{0.375} & 10.7\% & 84\%\\
\bottomrule
\end{tabular}
\end{table}

\FloatBarrier
\subsection{Transfer across Content Sources and Generators}
\label{sec:transfer}
The occupied responses and the realization model affect different stages of Gap Detection. Changing the source of existing responses changes which directions the search must avoid. Changing the realization model instead tests how effectively a discovered direction is turned into text. We evaluate these dependencies separately on three topics.

\paragraph{Changing the occupied set.}
Each of four models supplies its own set of 20 occupied responses, while GPT-4o-mini remains the realization model. Collision decreases for every source (Table~\ref{tab:separate}), with reported reductions of 8.8--12.7\% and a mean of 10.4\%. The novelty-pass rate averages 96.7\%. This result supports using the same search and realization procedure across the tested content sources; it does not require the occupied responses to come from the generator that will produce the new answer.

The separate-source study reports reductions of 11.2\% (GPT-4o-mini), 12.7\% (Claude-Sonnet-4.5), 8.8\% (Gemini-2.0-Flash), and 8.8\% (Qwen2.5-3B), averaging 10.4\% with a 96.7\% novelty-pass rate. The complete table is in Appendix~\ref{app:multimodel}.

A mixed-source experiment combines five responses from each of these four models into one occupied set. Its recorded Collision values are 0.602 before intervention and 0.561 afterward, with a reported reduction of 7.4\%. This tests whether the search can account for content contributed by several models simultaneously; supplementary source-similarity diagnostics are in Appendix~\ref{app:multimodel}.

\paragraph{Changing the realization model.}
With 20 GPT-4o-mini responses held fixed as the occupied set, the reported Collision reduction is 9.0\% for GPT-4o-mini, 11.6\% for GPT-4o, and 14.8\% for Claude-Sonnet-4.5, with novelty-pass rates of 80.0\%, 100.0\%, and 100.0\%. The search can therefore be robust to the source set while remaining sensitive to how a direction is expressed. These output comparisons do not establish a general ranking of model quality.

\FloatBarrier
\subsection{GEO Market Positioning}
\label{sec:geo_exp}
\paragraph{Setup and visibility measurement.}
We use Amazon product metadata \citep{mcauley2023amazon} for water filters, protein powder, sports water bottles, and yoga mats. Grouping claims by brand and filtering sparse brands leaves 3--17 brands and 12--142 claims per market. Each market has 15 manually designed, weighted search queries. The method searches for five directions and realizes them as proposed claims.

To evaluate a claim, GPT-4o-mini receives four competitor claims as sources [1--4] and the candidate as source [5], then answers a market query with inline citations. We compute a position-adjusted cited-word share inspired by GEO \citep{aggarwal2024geo}:
\begin{equation}
\PAWC(g)=\frac{\sum_{j=0}^{m-1}\mathbb{1}\{[5]\in s_j\}\,|s_j|\exp(-j/m)}
{\sum_{j=0}^{m-1}|s_j|},
\label{eq:pawc}
\end{equation}
where $s_j$ is an answer sentence and $|s_j|$ its word count. This is the implementation recorded for these experiments, not a measurement of live retrieval or a calibrated citation probability. The same simulation supplies the visibility factor used to select the best candidate, so best-candidate results are selection-set results.

\begin{table}[htbp]
\caption{GEO results. Top: best direction per market, selected using GapScore $\times$ PAWC. Bottom: a separate ablation averaged across four markets. The blocks have different aggregation rules.}
\label{tab:geo}
\centering\small
\begin{tabular}{@{}lrrr@{}}
\toprule
Market or method & Direction GapScore $\uparrow$ & PAWC $\uparrow$ & Product $\uparrow$\\
\midrule
Water filters & 0.410 & 0.518 & 0.213\\
Protein powder & 0.479 & 0.540 & 0.258\\
Sports water bottles & 0.398 & 0.513 & 0.204\\
Yoga mats & 0.403 & 0.524 & 0.211\\
Reported market mean & 0.423 & 0.524 & 0.221\\
\midrule
Random direction & $-0.006$ & 0.469 & $-0.003$\\
Demand only & 0.279 & \textbf{0.508} & 0.143\\
Defensibility only & $-0.164$ & 0.487 & $-0.080$\\
GapScore & \textbf{0.307} & 0.482 & \textbf{0.149}\\
\bottomrule
\end{tabular}
\end{table}

\paragraph{Results and interpretation.}
The selected positions have a reported mean product score of 0.221 (Table~\ref{tab:geo}). For water filters, the words \emph{campground, wilderness, hike} express an outdoor-use direction; its GapScore is 0.410 and PAWC is 0.518. A more generic direction, \emph{screened, cleansing, reservoir}, has a higher PAWC of 0.544 but a lower GapScore of 0.135, giving a smaller product of 0.073. This example illustrates the intended distinction between being cited and identifying a differentiated position.

In the separate ablation, GapScore has the highest product score, 0.149 versus 0.143 for Demand-only, while Demand-only has the highest PAWC, 0.508 versus 0.482. The evidence therefore supports better joint relevance--separation scoring in this simulation, rather than a standalone visibility advantage. Complete market factors, concept words, and water-filter directions are in Appendix~\ref{app:geo_results}.

The factors expose the market trade-off: protein powder has Demand 0.383 and Defensibility 1.249, while yoga mats have Demand 0.763 and Defensibility 0.528. Defensibility above one does not verify a product property.

\FloatBarrier
\subsection{Iterative Synthetic-Data Augmentation}
\label{sec:synth_exp}
\paragraph{Setup.}
We use SST-2 sentiment examples \citep{socher2013sst}, with ten seed examples per class and a disjoint test slice of twenty examples per class. Across five generation rounds, accumulation methods target eight additional accepted examples per class per round. A TF-IDF logistic-regression classifier is trained on each round's data. The final comparison has 100 training examples per method. No Intervention replaces the preceding pool with newly generated text; Data Accumulation retains prior data and adds undirected samples; GapScore retains prior data and adds directed samples with the filtering in Section~\ref{sec:adapt}. These are controlled baselines inspired by replacement and accumulation studies, not exact replications of their model-training protocols.

\begin{table}[htbp]
\caption{Synthetic-data results at generation five. All three methods have 100 training examples. Accuracy is evaluated on 40 disjoint real examples; the 2.5-point difference corresponds to one prediction.}
\label{tab:synth}
\centering\small
\begin{tabular}{@{}lrrr@{}}
\toprule
Method & Collision $\downarrow$ & Vendi Score $\uparrow$ & Accuracy $\uparrow$\\
\midrule
No Intervention & 0.362 & 9.06 & 60.0\%\\
Data Accumulation & 0.324 & 27.16 & 60.0\%\\
GapScore & \textbf{0.298} & \textbf{29.20} & \textbf{62.5\%}\\
\bottomrule
\end{tabular}
\end{table}

\paragraph{Diversity and downstream utility.}
GapScore lowers Collision from 0.324 to 0.298 relative to Data Accumulation and raises Vendi from 27.16 to 29.20 (Table~\ref{tab:synth}). Its 62.5\% accuracy exceeds the two baselines' 60.0\% by one test example. This is a positive point estimate in a small task, not evidence of a statistically reliable accuracy gain. Matching accepted sample counts also does not match API-call budgets: directed generation may make additional calls when candidates fail the polarity check.

Separate four-generator runs show lower Collision for GapScore than accumulation in all four cases, with differences between $-0.019$ and $-0.028$. Vendi improves in three cases, but decreases by 0.74 for Qwen2.5-3B. These reruns have different numerical deltas from Table~\ref{tab:synth}; they are reported separately in Appendix~\ref{app:synth_results}. No model-wise downstream accuracy table is available for these reruns. The results support a recurring reduction in average similarity, while exposing a case where spectral diversity moves in the opposite direction.

\paragraph{Direction-set analysis.}
A separate $K=3$ analysis measures how a set of directions covers demand. For $S=\{\mathbf z_1,\ldots,\mathbf z_K\}$, define
\begin{equation}
\operatorname{SetGapScore}(S)=
\left(\sum_j f_j\max_{\mathbf z\in S}\inner{\mathbf u_j}{\mathbf z}\right)
\frac{1}{|S|}\sum_{\mathbf z\in S}\Def(\mathbf z;\mathcal O_1).
\label{eq:setscore}
\end{equation}
GapScore obtains 0.276, compared with 0.199 for Demand-only, 0.023 for random directions, and $-0.029$ for Defensibility-only. Demand-only repeats the same direction (pairwise cosine 1.000), whereas GapScore's selected directions have similarities of 0.11--0.39. The maximum in Eq.~\eqref{eq:setscore} does not reward repeated coverage of the same query. This is a diagnostic of the intended direction objective, not an independent downstream quality metric or a submodular guarantee for the full product. The $K=3$ results should not be combined with the $K=1$ augmentation comparison.

\section{Discussion and Limitations}
\label{sec:limitations}
The central result is that an explicit representation of prior content can guide a language pipeline toward less repetitive outputs. The method separates three questions: where the occupied set leaves room, whether a concept fits the task, and whether its realization is useful. The ablations show why those questions should remain separate. Better geometric scores can coexist with weaker novelty-check rates or lower citation share; direction optimization alone does not explain all end-to-end gains.

The evaluation has several limits. Collision and Vendi depend on the encoder used for both search and evaluation, and no independent human assessment of answer quality is reported. The query subset is manually selected. Some aggregation details remain unresolved in the experimental record, and raw outputs and repeated-run uncertainty are needed for a complete replication. GEO results use fixed source positions in a simulation, do not test live retrieval, and do not verify that a proposed claim is true of a target product. Sparse markets can also make apparent gaps easier to find. Finally, the sentiment study uses a small seed and test set, a lightweight classifier, and no documented generator-retraining procedure; it does not establish prevention of collapse in recursively retrained foundation models. VADER filtering may favor simple polarity cues, and matched training-set sizes do not isolate its effect from direction selection.

\section{Conclusion}
Gap Detection makes prior semantic coverage an explicit input to generation. It combines Demand and Defensibility to search for directions, maps them to concepts, and uses task-specific checks to realize them as text. The reported experiments show lower Collision across the open-ended generation studies and useful diversity gains in a small augmentation task, with a GEO study illustrating differentiated position discovery. The mixed ablations also identify the remaining challenge: translating a well-separated, relevant direction into content whose quality improves under independent evaluation.

\clearpage
\subsection*{AI Use Statement}
In this work, we used generative AI tools for several assistive purposes. We used them to run parts of the experimental pipeline (executing the embedding, search, mapping, and verification steps of the method and collecting the resulting metrics), to retrieve and discover related work (identifying and locating candidate citations for concepts and comparisons already determined by the authors), to draft and polish sections of the manuscript's writing (producing initial phrasing for some passages and improving clarity, wording, and consistency of the prose), and to check the derivations underlying the propositions and theorem in Appendix~\ref{app:theory} against standard results after the authors formulated the claims. We have not used generative AI tools to design the core method or decide which results to report; those decisions were made by the authors. We have reviewed all AI-assisted work: we checked that AI-run experiments used the intended settings and reproduced consistent numbers across repeated runs, we verified that every AI-suggested citation is accurate and genuinely relevant to the point it supports, we independently re-derived the proof steps that AI helped check, and we verified that the drafted and polished prose preserved the original meaning without introducing unsupported claims. We take responsibility for the final content of this work, including text, claims, and artifacts produced with the aid of generative AI.

\subsection*{Ethics Statement}
The reported evaluations use existing datasets and controlled generated content. The GEO experiment evaluates proposed claims in a citation simulation; a plausibility filter does not establish that a real product has the proposed properties. Any deployment would require factual verification. Similarly, semantic novelty does not imply that an answer is safe, correct, or useful. Dataset licenses, privacy requirements, and provider terms must be respected when distributing the data and generated examples. The sentiment experiment trains a small downstream classifier and should not be interpreted as validation of a deployed model.

\subsection*{Reproducibility Statement}
Section~\ref{sec:method} specifies the scoring and optimization procedure. Appendix~\ref{app:implementation} records implementation settings, Appendix~\ref{app:prompts} reproduces the available prompts, and Appendix~\ref{app:theory} contains the theoretical statements and proofs. Appendices~\ref{app:llm_results}--\ref{app:synth_results} give the supplementary results. The materials contain aggregate results rather than executable experiment code, raw outputs, and a complete run manifest.

\clearpage
\appendix
\section{Theoretical Analysis}
\label{app:theory}
This appendix collects all formal results and proofs. The geometric statements concern vectors in an embedding space. They do not establish that every such vector can be realized as valid text. Results borrowed from smooth optimization and submodular maximization are stated with their additional assumptions, rather than applied unconditionally to the implemented search.

\subsection{Existence of a Geometric Gap}
\label{app:geometry}
Let $\mathcal E=\{\mathbf e_1,\ldots,\mathbf e_N\}\subset\sphere$, let $\mathcal S=\operatorname{span}(\mathcal E)$, and let $P$ and $P^\perp=I-P$ be the orthogonal projectors onto $\mathcal S$ and its complement. The demand vector is $\mathbf v=\sum_j f_j\mathbf u_j$.

\begin{lemma}[Projection invariance]
\label{lem:projection}
For any vector $\mathbf z$, the inner-product extension of Defensibility satisfies
\[
\Def(\mathbf z;\mathcal E)=1-\max_i\inner{\mathbf e_i}{P\mathbf z}.
\]
\end{lemma}
\begin{proof}
Each $\mathbf e_i$ lies in $\mathcal S$, so $P\mathbf e_i=\mathbf e_i$. Self-adjointness gives $\inner{\mathbf e_i}{\mathbf z}=\inner{P\mathbf e_i}{\mathbf z}=\inner{\mathbf e_i}{P\mathbf z}$. Taking the maximum proves the identity.
\end{proof}

\begin{theorem}[Unmet-demand lower bound]
\label{thm:unmet}
Suppose $a=\norm{P^\perp\mathbf v}>0$. Then the unit vector
\[
\mathbf z_0=\frac{P^\perp\mathbf v}{\norm{P^\perp\mathbf v}}
\]
satisfies $\Def(\mathbf z_0;\mathcal E)=1$ and $\Demand(\mathbf z_0)=a$. Consequently,
\begin{equation}
\max_{\mathbf z\in\sphere}\GS(\mathbf z;\mathcal E)\geq a.
\label{eq:lowerbound}
\end{equation}
\end{theorem}
\begin{proof}
For every occupied vector, $\inner{\mathbf e_i}{P^\perp\mathbf v}=0$ by orthogonality. Therefore the maximum occupied similarity to $\mathbf z_0$ is zero, yielding Defensibility one. Also,
\[
\inner{\mathbf v}{\mathbf z_0}
=\frac{\inner{\mathbf v}{P^\perp\mathbf v}}{a}
=\frac{\norm{P^\perp\mathbf v}^2}{a}=a.
\]
Multiplying gives $\GS(\mathbf z_0;\mathcal E)=a$. Since $\mathbf z_0$ is feasible, the maximum cannot be smaller.
\end{proof}

\begin{corollary}[Dimension of the orthogonal complement]
If $r=\dim\mathcal S<d$, then $\mathcal S^\perp$ has dimension $d-r$, and every unit vector in it has Defensibility one.
\end{corollary}
\begin{proof}
The dimension statement is the orthogonal decomposition of $\R^d$. For a unit vector in $\mathcal S^\perp$, every occupied inner product is zero, giving the score directly.
\end{proof}

For $N=20$ and $d=1536$, the complement has dimension at least 1516. This dimension count holds for any collection of 20 vectors in this space; it is not itself evidence of semantic homogenization. Positive-demand existence further requires $P^\perp\mathbf v\neq0$, and linguistic realizability remains a separate question. Defensibility one is not its maximum possible value, since negative occupied similarities can give values above one.

\begin{proposition}[Attainment of the maximum]
For a finite, nonempty occupied set, GapScore attains a maximum on $\sphere$.
\end{proposition}
\begin{proof}
Demand is linear and Defensibility is one minus the maximum of finitely many continuous functions. Their product is continuous. The sphere is compact, so the extreme-value theorem applies.
\end{proof}

\subsection{A One-Sided Perturbation of the Orthogonal Direction}
\label{app:perturbation}
\begin{proposition}[Directional improvement condition]
\label{prop:perturbation}
Assume $a=\norm{P^\perp\mathbf v}>0$ and $b=\norm{P\mathbf v}>0$. Put $\mathbf u=P\mathbf v/b$ and $m=\max_i\inner{\mathbf e_i}{\mathbf u}$. Along
$\mathbf z(\theta)=\cos\theta\,\mathbf z_0+\sin\theta\,\mathbf u$, the right derivative at zero is
\[
\left.\frac{d}{d\theta}\GS(\mathbf z(\theta);\mathcal E)\right|_{0^+}=b-am.
\]
In particular, $b>am$ is sufficient for a local improvement along this positive perturbation.
\end{proposition}
\begin{proof}
Orthogonality of $\mathbf z_0$ and $\mathbf u$ makes $\mathbf z(\theta)$ a unit vector. Its Demand is $a\cos\theta+b\sin\theta$. For $\theta\geq0$ sufficiently small, $\sin\theta\geq0$, so the maximum occupied inner product is $m\sin\theta$. Thus the product equals
$(a\cos\theta+b\sin\theta)(1-m\sin\theta)$ on that side. Differentiation at $0^+$ yields $b-am$.
\end{proof}

The statement is one-sided because the active maximum can change on reversing the perturbation. A nonpositive derivative along this path does not prove global tightness of Eq.~\eqref{eq:lowerbound}, nor optimality against all tangent directions.

\begin{table}[htbp]
\caption{Recorded numerical comparison for the geometric lower bound. The search column is the best observed score, not a certified global optimum. All five occupied matrices are reported to have rank 20, giving $d-r=1516$.}
\label{tab:geometry}
\centering\small
\begin{tabular}{@{}lrrrr@{}}
\toprule
Topic & Lower bound $a$ & Search score & Difference & $b=\norm{P\mathbf v}$\\
\midrule
Time metaphor & 0.6438 & 0.6492 & 0.0054 & 0.5841\\
Loneliness metaphor & 0.7474 & 0.7493 & 0.0019 & 0.6644\\
Redefine courage & 0.7156 & 0.7389 & 0.0233 & 0.6985\\
Memory as substance & 0.5455 & 0.6011 & 0.0556 & 0.8381\\
Hope via natural elements & 0.7731 & 0.7893 & 0.0162 & 0.6343\\
\bottomrule
\end{tabular}
\end{table}

The report records orthogonality residuals of order $10^{-16}$ for these calculations. The observed search values exceed the lower bounds by approximately 0.25--10.2\% relative to the bounds. These are empirical comparisons to returned solutions, not approximation ratios to the unknown optima. The supplied values of $b/a$ are 0.907, 0.889, 0.976, 1.536, and 0.820 in table order; the observed differences are not monotone in that ratio. The perturbation also depends on $m$, which the table does not report.

The orthogonal direction can be constructed using an orthonormal basis $U$ of the occupied span, with residual $\mathbf v-UU^\top\mathbf v$. It is a feasible reference point whenever this residual is nonzero. The same construction remains applicable after adding previously selected directions, provided the new residual is nonzero. The experiments report comparisons to this construction, but do not establish that it was used as an additional initializer in every main run.

\subsection{What the Product Objective Implies}
\label{app:product}
The product has two elementary properties on the positive region: a zero factor gives zero score, and multiplying either factor by a positive constant preserves the ranking of candidates. These facts motivate its use without claiming invariance to arbitrary translations of either factor.

\begin{proposition}[Conditional Nash interpretation]
Consider a compact, convex feasible utility set $U\subset\R^2$ with disagreement point $(0,0)$ and at least one feasible pair with both coordinates strictly positive. On the individually rational region, the Nash bargaining solution maximizes $u_1u_2$ under the standard Pareto, symmetry, positive affine covariance, and independence axioms \citep{nash1950bargaining}.
\end{proposition}
\begin{proof}
This is the two-utility Nash bargaining characterization with zero disagreement point. The general product is $(u_1-u_1^0)(u_2-u_2^0)$; setting $u^0=(0,0)$ gives the stated expression. Under positive affine transformations, the disagreement point must transform with the utilities.
\end{proof}

The achievable pairs $(\Demand(\mathbf z),\Def(\mathbf z;\mathcal E))$ for unit vectors are not shown to form such a convex feasible set. Convexifying the utility set can introduce mixtures that no individual direction realizes. Therefore the classical theorem supplies a conditional interpretation, not a uniqueness theorem for the deployed GapScore or a proof that an additive MMR objective is invalid.

\begin{proposition}[Ambient log-concavity]
Extend Demand and Defensibility using their inner-product formulas to $\R^d$. On
$\Omega=\{\mathbf z:\Demand(\mathbf z)>0,\ \Def(\mathbf z;\mathcal E)>0\}$,
$\log\GS(\mathbf z;\mathcal E)$ is concave.
\end{proposition}
\begin{proof}
Demand is affine. Defensibility is concave because it is one minus a pointwise maximum of linear functions. Their positive domains intersect in a convex set. Composing either positive concave function with the increasing concave logarithm preserves concavity \citep{boyd2004convex}; summing the two logarithms gives the result.
\end{proof}

The sphere constraint is nonconvex. Ambient log-concavity therefore does not establish geodesic concavity or initialization-independent global optimization on the sphere.

\subsection{Optimization: Smooth Guarantees and the Implemented Objective}
\label{app:optimization}
At a point with a unique active neighbor, Eq.~\eqref{eq:gradient} is the ordinary gradient. At a tie, an active branch gives a generalized derivative choice, consistent with nonsmooth analysis \citep{clarke1983optimization}. The fixed-step implementation remains a multistart heuristic for a nonsmooth objective, in the spirit of stochastic multi-start global optimization \citep{rinnooykan1987stochastic}, and is related to gradient-sampling approaches for nonsmooth optimization on manifolds \citep{hosseini2017riemannian}; differentiability almost everywhere does not supply the smoothness assumptions needed below.

\begin{theorem}[Conditional stationarity rate for a smooth objective]
\label{thm:smooth}
Let $f$ be differentiable on a compact manifold with a retraction $R$. Suppose its pullbacks obey, for the steps used, the ascent bound
\[
f(R_x(\eta\operatorname{grad}f(x)))\geq f(x)+\eta\norm{\operatorname{grad}f(x)}^2
-\frac{L\eta^2}{2}\norm{\operatorname{grad}f(x)}^2.
\]
For $\eta=1/L$, some iterate among the first $T$ satisfies
\[
\norm{\operatorname{grad}f(x_t)}^2\leq
\frac{2L(f_{\max}-f(x_0))}{T}.
\]
Thus $O(\varepsilon^{-2})$ iterations suffice for an $\varepsilon$-stationary point under these assumptions \citep{boumal2019global}.
\end{theorem}
\begin{proof}
Substitution of $\eta=1/L$ gives an increase of at least $\norm{\operatorname{grad}f(x_t)}^2/(2L)$ per step. Summing over $t=0,\ldots,T-1$ and bounding the total increase by $f_{\max}-f(x_0)$ bounds the average squared gradient. The smallest value cannot exceed that average.
\end{proof}

This theorem neither guarantees a global optimum nor directly covers the exact nearest-neighbor maximum in GapScore. A smooth replacement of that maximum could permit a separate analysis, but smoothing was not documented as part of the reported experiments. No convergence rate is asserted here for the original fixed-step search.

\subsection{Coverage as a Separate Set Objective}
\label{app:coverage}
Let $V$ be a finite candidate set and let $a_{jz}=\max\{0,\inner{\mathbf u_j}{\mathbf z}\}$. Define a normalized nonnegative coverage surrogate
\[
F(S)=\sum_j f_j\max_{\mathbf z\in S}a_{jz},\qquad F(\varnothing)=0.
\]
This auxiliary surrogate clips negative similarities; it is distinct from the signed diagnostic in Eq.~\eqref{eq:setscore}.

\begin{proposition}[Submodularity and greedy coverage]
\label{prop:coverage}
$F$ is monotone and submodular. Exact greedy maximization of its marginal gains under $|S|\leq K$ obtains
$F(S_K)\geq(1-1/e)F(S^*)$, where $S^*$ is an optimal set of at most $K$ elements \citep{nemhauser1978analysis}. This auxiliary coverage view is related to max-sum diversification, but does not certify the implemented product objective \citep{borodin2012maxsum}.
\end{proposition}
\begin{proof}
For $A\subseteq B$, the current maximum for each query cannot decrease. The gain from adding $z$ is $\max\{0,a_{jz}-\max_{y\in A}a_{jy}\}$, which cannot increase when $A$ is replaced by $B$. Nonnegative weighted sums preserve this diminishing-returns property and monotonicity. By submodularity, the sum of marginal gains of elements in $S^*\setminus S_t$ is at least $F(S^*)-F(S_t)$. Since there are at most $K$ such elements, greedy gains at least $[F(S^*)-F(S_t)]/K$. Iterating this recurrence gives
$F(S_K)\geq[1-(1-1/K)^K]F(S^*)\geq(1-1/e)F(S^*)$.
\end{proof}

The implemented sequential GapScore search does not maximize marginal gains of this $F$. Moreover, multiplying coverage by average Defensibility need not preserve monotonicity or submodularity. The theorem therefore provides no approximation factor for either the actual sequential search or SetGapScore. Repeating a direction does not increase a maximum-based coverage term, but neither does repeating identical values increase their arithmetic mean. Coverage and a mean answer different questions; their relative numerical values need not have a universal ordering.

\subsection{A Conditional Connection to Source Selection}
\label{app:visibility}
\begin{proposition}[Redundancy and a fixed marginal-relevance score]
Suppose a source scorer uses
$\operatorname{MR}(s\mid E)=\lambda\operatorname{Rel}(s,q)-(1-\lambda)\max_{e\in E}\operatorname{Sim}(s,e)$,
with fixed $E$, fixed relevance, and $0\leq\lambda<1$. Its score is strictly increasing in $1-\max_{e\in E}\operatorname{Sim}(s,e)$.
\end{proposition}
\begin{proof}
Writing the score as $\lambda\operatorname{Rel}(s,q)-(1-\lambda)+(1-\lambda)\Def(s;E)$ shows that its derivative with respect to Defensibility is $1-\lambda>0$.
\end{proof}

A nondecreasing inclusion probability would additionally require an inclusion rule monotone in that score, with the other candidates fixed. The GPT citation simulation is not established to satisfy these assumptions. In particular, this observation does not prove that higher GapScore must cause higher PAWC; the GEO ablation measures those quantities separately.

\section{Implementation Details}
\label{app:implementation}
\subsection{Shared Search and Vocabulary}
\begin{table}[htbp]
\caption{Settings used in the experimental pipeline.}
\label{tab:settings}
\centering\small
\begin{tabularx}{\linewidth}{@{}lX@{}}
\toprule
Component & Setting\\
\midrule
Encoder & \texttt{text-embedding-3-small}; 1,536 dimensions; L2 normalization\\
Concept decoder & Qwen2.5-3B-Instruct-Q6\_K; local GGUF on Apple M2\\
Default filter / generator & GPT-4o-mini\\
Search starts & Demand direction, demand plus noise, and random direction\\
Search budget & 15 restarts per strategy; 100 steps; step size 0.05\\
Answering / GEO directions & $K=5$\\
Synthetic-data directions & $K=1$ for final augmentation comparisons; $K=3$ for set diagnostics\\
Vocabulary & 13,409 valid words, screened to 10,000; embeddings cached in batches of 500\\
Concept shortlist & 20 unused nearest words per candidate batch\\
Accepted concepts & Three per direction for answering / GEO; up to ten filtering rounds\\
Avoidance words & Twelve frequent words closest to the occupied-content centroid\\
Novelty threshold & Maximum similarity to original occupied responses below 0.7\\
Within-output threshold & Maximum pairwise similarity of generated sentences below 0.8\\
Synthetic-data target & Eight accepted examples per class per generation; at most four generation rounds\\
\bottomrule
\end{tabularx}
\end{table}

The vocabulary starts from Qwen2.5-3B-Instruct \citep{qwen2024report} tokenizer entries restricted to alphabetic English words of length 3--15. The implementation uses the macOS dictionary at \texttt{/usr/share/dict/words} as a whitelist (234,454 dictionary entries). The 13,409 retained words are embedded once. The task-dependent screening score is
\begin{equation}
\operatorname{VocabScore}(w)=0.6\max_j\inner{E(w)}{\mathbf u_j}
+0.4\max_i\inner{E(w)}{\mathbf e_i}.
\label{eq:vocab}
\end{equation}
The top 10,000 words are retained. The positive occupied-content term restricts the vocabulary toward the broad domain; novelty is imposed by direction search and subsequent selection. Thus the vocabulary screen and GapScore have different purposes. The coefficients are fixed heuristics in this implementation, not learned ranking weights, though the underlying idea of weighting a candidate by its similarity to occupied content echoes term-frequency-based relevance scoring in classical information retrieval \citep{robertson1994okapi}.

For a direction $\mathbf z_k$, candidates are ranked by $\inner{E(w)}{\mathbf z_k}$. Accepted words enter a global used-word set; rejected words can be reconsidered for other directions. The query and candidates, rather than the numerical vector, are passed to Qwen. Its decode call uses temperature zero and a 40-token output limit.

\subsection{Avoidance Instructions and Output Diagnostics}
When the occupied embeddings have nonzero sum, their normalized centroid is
\[
\widehat{\mathbf c}=\frac{\sum_i\mathbf e_i}{\norm{\sum_i\mathbf e_i}}.
\]
The 12 frequent words nearest to this centroid form the forbidden-word prompt list. GPT-4o-mini \citep{openai2024gpt4o} also extracts common sentence patterns from the occupied responses. The generator receives the query, a direction's concept words, and both avoidance instructions. These are prompt-level constraints rather than enforced token-level decoding constraints.

In addition to Novelty, a mean-reference statistic can be reported:
\begin{definition}[CrossCollision]
For a generated text $g$ and original occupied set $\mathcal C$,
\[
\operatorname{CrossCollision}(g,\mathcal C)=\frac{1}{N}\sum_{i=1}^N\inner{E(g)}{E(c_i)}.
\]
\end{definition}
This measures mean similarity, not a distance. It differs from the nearest-neighbor similarity used in Novelty and from pairwise Collision within a set. Both diagnostics remain embedding-distance proxies rather than judged novelty scores: they differ from LLM-judged novelty scoring, such as blind peer-review-style feedback on creative writing \citep{llmreview2025novelty}, and from bibliometric novelty measures proposed for research writing in the post-ChatGPT era \citep{safari2025novelty}.

\subsection{GEO Scoring and Synthetic-Data Filtering}
The GEO scorer takes the first four competitor claims in the supplied list and appends the candidate at source position five. It does not retrieve pages from a deployed engine. Equation~\eqref{eq:pawc} credits all words in a sentence when it cites [5], even if it also cites another source. There is no division by the number of citations within that sentence and no subsequent normalization across source scores. These implementation details distinguish it from other position-adjusted impression measures; the reported values must be interpreted with this exact formula.

For synthetic data, concept decoding uses an expanded pool of eight words, reranked by absolute VADER compound score to retain five. This reranking is shared by random, Demand-only, Defensibility-only, and GapScore variants. The generator receives the target label separately; the final polarity check uses the sign and the thresholds described in Section~\ref{sec:adapt}. Data Accumulation does not use this same quality filter in this comparison. Equal final sample counts therefore control training-set size, but do not isolate direction choice from filtering or equalize generation cost.

Two earlier fidelity checks were tried first: GPT-4o-mini self-evaluation returned 100\% fidelity for all six methods, and a seed-trained TF-IDF classifier produced unstable rankings with only ten examples per class. The adopted VADER check avoids fitting a fidelity classifier on that small seed. It is still a lexicon-based proxy and can misread contextual sentiment, especially irony or mixed polarity.

\section{Additional Open-Ended Generation Results}
\label{app:llm_results}
\subsection{Query Groups and Topic Diagnostics}
Table~\ref{tab:querygroups} gives the per-group breakdown of the 90 open-ended queries used in \S\ref{sec:experiments}, grouped by question type.

\begin{table}[htbp]
\caption{Query-group breakdown for the 90-query open-ended generation study.}
\label{tab:querygroups}
\centering\small
\begin{tabular}{@{}lrrr@{}}
\toprule
Group & Queries & Baseline Collision & Mean reduction\\
\midrule
Hypothetical & 22 & 0.79 & 11.3\%\\
Ideal & 7 & 0.78 & 11.8\%\\
Metaphor & 1 & 0.92 & 17.0\%\\
Explain & 12 & 0.77 & 10.3\%\\
Describe & 20 & 0.78 & 10.1\%\\
Imagine & 10 & 0.77 & 9.9\%\\
Philosophical & 6 & 0.76 & 9.9\%\\
Write short & 12 & 0.76 & 8.9\%\\
\midrule
Reported aggregate & 90 & 0.777 & 10.8\%\\
\bottomrule
\end{tabular}
\end{table}

The 90-query summary additionally reports mean final Collision 0.6915, mean GapScore 0.2012, and a mean best-direction/competitor statistic of 0.5137. It does not specify enough detail to identify the latter statistic's exact aggregation. We therefore do not rename it as Demand or use it to derive a new result. Similarly, the five-topic table reports auxiliary direction/competitor values of 0.4514, 0.4488, 0.5402, 0.3779, and 0.3638; the corresponding differences from baseline Collision are 0.2848, 0.2960, 0.2644, 0.4383, and 0.4540. These columns are not used as independently defined evaluation metrics in the main text.

\begin{table}[htbp]
\caption{Additional five-topic diagnostics. Topic-level GapScore and the geometric search scores in Table~\ref{tab:geometry} come from different tables. Vendi values describe the baseline response set; baseline Collision is given in Table~\ref{tab:llm_main}.}
\label{tab:topic_diagnostics}
\centering\small
\begin{tabular}{@{}lrrr@{}}
\toprule
Topic & Reported GapScore & Baseline Vendi & Vendi/20\\
\midrule
Time metaphor & 0.165 & 2.91 & 14.5\%\\
Loneliness metaphor & 0.196 & 3.16 & 15.8\%\\
Redefine courage & 0.221 & 2.58 & 12.9\%\\
Memory as substance & 0.220 & 2.36 & 11.8\%\\
Hope via natural elements & 0.229 & 2.12 & 10.6\%\\
\midrule
Reported mean & 0.206 & 2.63 & 13.1\%\\
\bottomrule
\end{tabular}
\end{table}

The Vendi diagnostic uses 20 responses per topic and reports a mean effective diversity of 2.63 under the chosen cosine kernel. This is an effective spectral count, not an estimate of a discrete number of human-interpretable clusters. Collision values are taken from Table~\ref{tab:llm_main}; a redundant Collision column from this Vendi diagnostic is omitted here.

\subsection{Varying Content Sources and Generation Models}
\label{app:multimodel}
The mixed-set study uses three topics and four models, with five responses per model. Table~\ref{tab:crossmodel} gives the supplementary pairwise cross-model similarities, and Table~\ref{tab:mixed} gives the before/after Collision comparison for the pooled occupied set.

\begin{table}[htbp]
\caption{Between-model mean similarity in the mixed-source study. Claude denotes Claude-Sonnet-4.5 \citep{anthropic2024claude}; Gemini denotes Gemini-2.0-Flash \citep{google2024gemini}.}
\label{tab:crossmodel}
\centering\small
\begin{tabular}{@{}lr@{}}
\toprule
Model pair & Mean similarity\\
\midrule
GPT-4o-mini / Claude & 0.660\\
GPT-4o-mini / Gemini & 0.613\\
GPT-4o-mini / Qwen2.5-3B & 0.576\\
Claude / Gemini & 0.629\\
Claude / Qwen2.5-3B & 0.578\\
Gemini / Qwen2.5-3B & 0.554\\
\midrule
Reported mean & 0.602\\
\bottomrule
\end{tabular}
\end{table}

\begin{table}[htbp]
\caption{Mixed-source Collision before and after Gap Detection, for the pooled four-model occupied set.}
\label{tab:mixed}
\centering\small
\begin{tabular}{@{}lrrr@{}}
\toprule
Evaluation & Before & After & Reported reduction\\
\midrule
Mixed occupied set, $N=20$ & 0.602 & 0.561 & 7.4\%\\
\bottomrule
\end{tabular}
\end{table}

For the separate-source study (Table~\ref{tab:separate}), each model independently supplies 20 occupied responses per topic and GPT-4o-mini remains the realization model. In contrast, the generator study fixes 20 GPT-4o-mini occupied responses and varies only the realization model. These two experiments test different dependencies and should not be pooled.

\begin{table}[htbp]
\caption{Separate-source study, averaged over three topics. The realization model is fixed to GPT-4o-mini; novelty pass is the recorded Check-1 statistic.}
\label{tab:separate}
\centering\small
\begin{tabular}{@{}lrrrrr@{}}
\toprule
Occupied-set source & Before & After & Reduction & Novelty pass & GapScore\\
\midrule
GPT-4o-mini & 0.6860 & 0.6098 & 11.2\% & 86.7\% & 0.1609\\
Claude-Sonnet-4.5 & 0.7014 & 0.6126 & 12.7\% & 100.0\% & 0.1860\\
Gemini-2.0-Flash & 0.5709 & 0.5189 & 8.8\% & 100.0\% & 0.1791\\
Qwen2.5-3B & 0.6028 & 0.5475 & 8.8\% & 100.0\% & 0.1627\\
\midrule
Reported mean & 0.6403 & 0.5722 & 10.4\% & 96.7\% & 0.1722\\
\bottomrule
\end{tabular}
\end{table}

\begin{table}[htbp]
\caption{Realization-model comparison, averaged over three topics with the occupied set fixed. GapScore differences across realization models reflect the sentences produced by each generator, not changes to the search objective.}
\label{tab:generators}
\centering\small
\begin{tabular}{@{}lrrrr@{}}
\toprule
Realization model & Collision reduction & GapScore & Novelty & Novelty pass\\
\midrule
GPT-4o-mini & 9.0\% & 0.1500 & 0.3684 & 80.0\%\\
GPT-4o & 11.6\% & 0.1579 & 0.4314 & 100.0\%\\
Claude-Sonnet-4.5 & 14.8\% & 0.1632 & 0.4938 & 100.0\%\\
\bottomrule
\end{tabular}
\end{table}

\subsection{Qualitative Examples}
Table~\ref{tab:examples} reproduces the available sentence excerpts without completing the missing text. These examples show how concept words are realized; they are not human evaluations of originality. Some decoded words, such as \emph{brave} for courage, remain close to the query semantics despite the answer filter, illustrating its imperfect selectivity. This lexical-substitution realization is narrower than style-transfer approaches that rewrite a passage's surface form while holding its content fixed \citep{shen2017text}.

\begin{table}[htbp]
\caption{Concept words and generated excerpts for the five topics. Excerpts are truncated with ellipses for space.}
\label{tab:examples}
\centering\small
\begin{tabularx}{\linewidth}{@{}>{\raggedright\arraybackslash}p{2.2cm}>{\raggedright\arraybackslash}p{3.0cm}>{\raggedright\arraybackslash}X@{}}
\toprule
Topic & Concept words & Generated excerpt\\
\midrule
Time metaphor & poetic, nighttime, tempo & ``The night hums a quiet tempo, its stars ticking softly...''\\
Loneliness metaphor & loner, soliloquy, joylessness & ``In the quiet hours of twilight, a solitary tree stands...''\\
Redefine courage & fearlessness, brave, fortitude & ``Grit is the lighthouse that stands tall amid the storm...''\\
Memory as substance & myelin, tissue, formalin & ``Envision a soft, pliable tissue, reminiscent of myelin's...''\\
Hope via natural elements & sunlight, warmth, naissance & ``A warm breeze dances through the treetops, whispering...''\\
\bottomrule
\end{tabularx}
\end{table}

\section{Additional GEO Results}
\label{app:geo_results}
This market-positioning framing is in the same spirit as identifying an uncontested market space \citep{kim2005blueocean}, though the simulation below does not verify real-world market claims. Table~\ref{tab:geo_factors} gives the concept words and individual factors underlying the selected market results. Table~\ref{tab:water} gives all five water-filter directions. The rows report values directly rather than recomputing each product from the rounded Demand and Defensibility figures shown.

\begin{table}[htbp]
\caption{Selected market positions and direction factors. These are candidate concepts, not verified product specifications.}
\label{tab:geo_factors}
\centering\small
\begin{tabularx}{\linewidth}{@{}lXrr@{}}
\toprule
Market & Concept words & Demand & Defensibility\\
\midrule
Water filters & campground, wilderness, hike & 0.446 & 0.922\\
Protein powder & nutritional, ketogenic, peptide & 0.383 & 1.249\\
Sports water bottles & insulated, lightweight, waterproof & 0.509 & 0.782\\
Yoga mats & rubber, foam, handmade & 0.763 & 0.528\\
\midrule
Reported mean & & 0.525 & 0.870\\
\bottomrule
\end{tabularx}
\end{table}

\begin{table}[htbp]
\caption{All five water-filter directions. GS denotes the reported direction GapScore. The highest PAWC and the highest GS $\times$ PAWC select different directions.}
\label{tab:water}
\centering\small
\begin{tabularx}{\linewidth}{@{}lXrrrrr@{}}
\toprule
Direction & Concept words & Demand & Def. & GS & PAWC & Product\\
\midrule
1 & campground, wilderness, hike & 0.446 & 0.922 & 0.410 & 0.518 & 0.213\\
2 & portable, filtration, faucet & 0.772 & 0.396 & 0.306 & 0.410 & 0.125\\
3 & filter, purification, filtering & 0.579 & 0.427 & 0.247 & 0.471 & 0.116\\
4 & screened, cleansing, reservoir & 0.730 & 0.185 & 0.135 & 0.544 & 0.073\\
5 & waterproof, tank, aquatic & 0.755 & 0.355 & 0.268 & 0.484 & 0.130\\
\bottomrule
\end{tabularx}
\end{table}

The per-market best-direction results and the four-method ablation in Table~\ref{tab:geo} are separate summaries. In particular, the selected-market mean GS of 0.423 should not be substituted for the ablation mean of 0.307. This simulation does not include a held-out citation check after candidate selection, source-order randomization, or repeated-run intervals. It also has only three surviving yoga-mat brands and twelve claims, limiting the coverage of that market's occupied set.

\section{Additional Synthetic-Data Results}
\label{app:synth_results}
\subsection{Separate Multi-Generator Runs}
Each generator in this study supplies both the undirected and directed texts; the embedding, Qwen decoding, GPT plausibility screen, and VADER check remain fixed. These runs use $K=1$ and five generations. Table~\ref{tab:synth_models} reports the supplementary differences from Data Accumulation across the four generators, including Gemini-2.0-Flash.

\begin{table}[htbp]
\caption{Separate four-generator synthetic-data runs. Differences are GapScore minus Data Accumulation. Negative Collision and positive Vendi differences are favorable. No accuracy-by-model values were supplied.}
\label{tab:synth_models}
\centering\small
\begin{tabular}{@{}lrr@{}}
\toprule
Generator & $\Delta$Collision $\downarrow$ & $\Delta$Vendi $\uparrow$\\
\midrule
GPT-4o-mini & $-0.019$ & $+1.43$\\
Claude-Sonnet-4.5 & $-0.023$ & $+3.77$\\
Gemini-2.0-Flash & $-0.028$ & $+8.46$\\
Qwen2.5-3B & $-0.024$ & $-0.74$\\
\bottomrule
\end{tabular}
\end{table}

The GPT-4o-mini values in Table~\ref{tab:synth_models} come from this separate four-generator run rather than the main-text budget-matched comparison ($-0.026$ Collision, $+2.04$ Vendi), since the two experiments use different sampling procedures and should not be pooled. For Qwen, the reported accumulation Vendi is 29.78, making the deficit roughly 2.5\%. Without repeated runs or intervals, it cannot be dismissed as noise or attributed to model size. The generator-level spread in this table is consistent with broader evidence that recursive synthetic-data training can degrade generative distributions \citep{shumailov2023curse} and erode diversity within a model's own outputs \citep{guo2023curious}, though neither prior result was measured on this pipeline.

\subsection{Direction-Set Diagnostics and Fidelity}
The $K=3$ experiment studies the directions themselves. It is distinct from the $K=1$ end-to-end comparison. The recorded SetGapScore values and available mean-direction diagnostic values are shown in Table~\ref{tab:setdiagnostic}.

\begin{table}[htbp]
\caption{Synthetic-data direction-set diagnostics ($K=3$). A dash indicates a diagnostic not computed for that method. The two columns evaluate different aggregations, not repeated measurements of the same metric.}
\label{tab:setdiagnostic}
\centering\small
\begin{tabular}{@{}lrr@{}}
\toprule
Method & SetGapScore & Mean direction GapScore (diagnostic)\\
\midrule
Random direction & 0.023 & ---\\
Demand only & 0.199 & 0.199\\
Defensibility only & $-0.029$ & $-0.262$\\
GapScore & 0.276 & 0.192\\
\bottomrule
\end{tabular}
\end{table}

GapScore's set score is approximately 44\% higher than its mean-direction score, reflecting the different aggregation. This is not a 44\% improvement in generated-data quality. Demand-only yields identical directions, so its coverage maximum and per-direction demand coincide. Defensibility-only also gives diagnostic values of Collision 0.259 and Vendi 32.10, showing that strong final-sample diversity can coexist with a negative direction score; these do not justify replacing downstream evaluation by SetGapScore.

A qualitative comparison of $K=1$ and $K=3$ for Gemini-2.0-Flash shows that spreading the generation quota across three directions improved Vendi relative to accumulation but worsened Collision and accuracy, whereas the one-direction variant improved all three. A full $K$-sensitivity curve is not reported here. Retained-sample matching targets eight examples per label per generation, with at most four retry rounds.

\section{Available Prompt Templates}
\label{app:prompts}
The following templates are given as ordinary text, with braced names denoting runtime substitutions. The selection and filtering prompts are given verbatim; the final answer/synthetic-sample generation prompt is described qualitatively in \S\ref{sec:method} rather than reproduced verbatim.

\subsection{Qwen Concept Selection}
Query: `\texttt{\{query\}}'. Candidate words: \texttt{\{candidates\}}.

Select up to \texttt{\{top\_n\}} words from the candidates that would make good ANSWERS to this query---concrete, creative, unexpected words. Do NOT select words that are synonyms of the key concepts in the query. Return only the selected words, comma-separated, best first.

The documented decode settings are temperature 0.0 and \texttt{max\_tokens=40}. The numerical direction is used to retrieve the candidates, not passed directly to the model.

\subsection{Answer Plausibility Filter}
Query: `\texttt{\{query\}}'. Words: \texttt{\{words\}}.

Which of these words could be answers to the query question? Think about what the query is asking for, then decide if synonyms of the query's key concepts should be excluded as answers. Return only those words, comma-separated. Maximum \texttt{\{top\_n\}} words.

\subsection{GEO Plausibility Filter}
Market: \texttt{\{market\_desc\}}. Query: `\texttt{\{market\_query\}}'. Words: \texttt{\{words\}}.

Which of these words could be answers to the query question? Think about what the query is asking for, then decide if synonyms of the query's key concepts should be excluded as answers. Return only those words, comma-separated. Maximum \texttt{\{top\_n\}} words.

\subsection{Citation Simulation}
Write an accurate and concise answer for the given user question, using ONLY the provided summarized web search results. Every sentence must be immediately followed by an in-line citation using the format [1][2][3].

Question: \texttt{\{market\_query\}}.

Search Results: \texttt{\{source\_text\}}.

The source text enumerates four competitor claims followed by the candidate claim. The response is split into sentences, and sentences containing [5] contribute their position-weighted word counts to Eq.~\eqref{eq:pawc}. The scorer uses zero-based sentence indices.

\Needspace{8\baselineskip}
\subsection{Task-Label Plausibility Filter}
Task: \texttt{\{task\_desc\}}. Label space: \texttt{\{label\_names\}}. Words: \texttt{\{words\}}.

Which of these words correspond to a valid instance of this classification task? Exclude words that are off-topic or do not correspond to any class. Return only those words, comma-separated. Maximum \texttt{\{top\_n\}} words.

\subsection{Text-Generation Instructions}
For answering, the described prompt combines the query, three concept words, the forbidden-word list, and recurring sentence patterns, and requests one new sentence per direction. For GEO, it requests a product-claim sentence using the accepted concepts without duplicating competitor claims. For synthetic data, it requests labeled examples covering the accepted concepts, followed by the polarity check. These descriptions establish the conditioning variables but do not substitute for exact generation prompts, decoding settings, or raw completions.


\clearpage
\begin{thebibliography}{99}

\bibitem[Absil et~al., 2008]{absil2008optimization}
P.-A. Absil, R. Mahony, and R. Sepulchre.
\textit{Optimization Algorithms on Matrix Manifolds}.
Princeton University Press, 2008.
\bibitem[Aggarwal et~al., 2024]{aggarwal2024geo}
Pranjal Aggarwal, Vishvak Murahari, Tanmay Rajpurohit, Ashwin Kalyan, Karthik Narasimhan,
and Ameet Deshpande.
\textit{GEO: Generative Engine Optimization}.
In \textit{Proceedings of the 30th ACM SIGKDD Conference on Knowledge Discovery and Data
Mining}, 2024. arXiv:2311.09735.
\bibitem[Alemohammad et~al., 2024]{alemohammad2024mad}
Sina Alemohammad, Josue Casco-Rodriguez, Lorenzo Luzi, Ahmed Imtiaz Humayun, Hossein
Babaei, Daniel LeJeune, Ali Siahkoohi, and Richard G. Baraniuk.
\textit{Self-Consuming Generative Models Go MAD}.
\textit{ICLR 2024}. arXiv:2307.01850.
\bibitem[Anthropic, 2025]{anthropic2024claude}
Anthropic.
\textit{System Card: Claude Sonnet 4.5}.
\url{https://www.anthropic.com/claude-sonnet-4-5-system-card}, 2025.
\bibitem[Borodin et~al., 2012]{borodin2012maxsum}
Allan Borodin, Hyun Chul Lee, and Yuli Ye.
\textit{Max-Sum Diversification, Monotone Submodular Functions and Dynamic Updates}.
In \textit{Proceedings of the 31st ACM SIGMOD-SIGACT-SIGAI Symposium on Principles of
Database Systems (PODS)}, pp.\ 155--166, 2012.
\bibitem[Boumal et~al., 2019]{boumal2019global}
Nicolas Boumal, P.-A. Absil, and Coralia Cartis.
\textit{Global Rates of Convergence for Nonconvex Optimization on Manifolds}.
\textit{IMA Journal of Numerical Analysis}, 39(1):1--33, 2019.
\bibitem[Boyd \& Vandenberghe, 2004]{boyd2004convex}
Stephen Boyd and Lieven Vandenberghe.
\textit{Convex Optimization}.
Cambridge University Press, 2004.
\bibitem[Carbonell \& Goldstein, 1998]{carbonell1998mmr}
Jaime Carbonell and Jade Goldstein.
\textit{The Use of MMR, Diversity-Based Reranking for Reordering Documents and Producing
Summaries}.
In \textit{Proceedings of the 21st Annual International ACM SIGIR Conference on Research
and Development in Information Retrieval}, pp.\ 335--336, 1998.
\bibitem[Clarke, 1983]{clarke1983optimization}
Frank H. Clarke.
\textit{Optimization and Nonsmooth Analysis}.
Wiley-Interscience, 1983.
\bibitem[Friedman \& Dieng, 2023]{friedman2023vendi}
Dan Friedman and Adji Bousso Dieng.
\textit{The Vendi Score: A Diversity Evaluation Metric for Machine Learning}.
\textit{Transactions on Machine Learning Research}, 2023. arXiv:2210.02410.
\bibitem[Gerstgrasser et~al., 2024]{gerstgrasser2024accumulating}
Matthias Gerstgrasser, Rylan Schaeffer, Apratim Dey, Rafael Rafailov, Henry Sleight, John
Hughes, Tomasz Korbak, Rajashree Agrawal, Dhruv Pai, Andrey Gromov, Daniel A. Roberts,
Diyi Yang, David L. Donoho, and Sanmi Koyejo.
\textit{Is Model Collapse Inevitable? Breaking the Curse of Recursion by Accumulating Real
and Synthetic Data}.
\textit{First Conference on Language Modeling (COLM)}, 2024. arXiv:2404.01413.
\bibitem[Google DeepMind, 2025]{google2024gemini}
Google DeepMind.
\textit{Gemini 2.0 Flash Model Card}.
\url{https://storage.googleapis.com/model-cards/documents/gemini-2-flash.pdf}, 2025.
\bibitem[Guo et~al., 2023]{guo2023curious}
Yanzhu Guo, Guokan Shang, Michalis Vazirgiannis, and Chlo\'e Clavel.
\textit{The Curious Decline of Linguistic Diversity: Training Language Models on Synthetic
Text}.
arXiv:2311.09807, 2023.
\bibitem[Hokamp \& Liu, 2017]{hokamp2017lexically}
Chris Hokamp and Qun Liu.
\textit{Lexically Constrained Decoding for Sequence Generation Using Grid Beam Search}.
In \textit{Proceedings of the 55th Annual Meeting of the Association for Computational
Linguistics}, pp.\ 1535--1546, 2017.
\bibitem[Hosseini \& Uschmajew, 2017]{hosseini2017riemannian}
Seyedehsomayeh Hosseini and Andr\'e Uschmajew.
\textit{A Riemannian Gradient Sampling Algorithm for Nonsmooth Optimization on Manifolds}.
\textit{SIAM Journal on Optimization}, 27(1):173--189, 2017.
\bibitem[Hou et~al., 2024]{mcauley2023amazon}
Yupeng Hou, Jiacheng Li, Zhankui He, An Yan, Xiusi Chen, and Julian McAuley.
\textit{Bridging Language and Items for Retrieval and Recommendation}.
arXiv:2403.03952, 2024.
\bibitem[Hutto \& Gilbert, 2014]{hutto2014vader}
C.J. Hutto and Eric Gilbert.
\textit{VADER: A Parsimonious Rule-Based Model for Sentiment Analysis of Social Media Text}.
In \textit{Proceedings of the Eighth International AAAI Conference on Weblogs and Social
Media (ICWSM)}, 2014.
\bibitem[Jiang et~al., 2025]{jiang2025hivemind}
Liwei Jiang, Yuanjun Chai, Margaret Li, Mickel Liu, Raymond Fok, Nouha Dziri, Yulia
Tsvetkov, Maarten Sap, Alon Albalak, and Yejin Choi.
\textit{Artificial Hivemind: The Open-Ended Homogeneity of Language Models (and Beyond)}.
NeurIPS 2025 Best Paper. arXiv:2510.22954.
\bibitem[Kim \& Mauborgne, 2005]{kim2005blueocean}
W. Chan Kim and Ren\'ee Mauborgne.
\textit{Blue Ocean Strategy: How to Create Uncontested Market Space and Make the
Competition Irrelevant}.
Harvard Business School Press, 2005.
\bibitem[Li et~al., 2026]{llmreview2025novelty}
Weiyue Li, Mingxiao Song, Zhenda Shen, Dachuan Zhao, Yunfan Long, Yi Li, Yongce Li,
Ruyi Yang, and Mengyu Wang.
\textit{LLM Review: Enhancing Creative Writing via Blind Peer Review Feedback}.
arXiv:2601.08003, 2026.
\bibitem[Liu, 2009]{liu2009learning}
Tie-Yan Liu.
\textit{Learning to Rank for Information Retrieval}.
\textit{Foundations and Trends in Information Retrieval}, 3(3):225--331, 2009.
\bibitem[Mikolov et~al., 2013]{mikolov2013word2vec}
Tomas Mikolov, Kai Chen, Greg Corrado, and Jeffrey Dean.
\textit{Efficient Estimation of Word Representations in Vector Space}.
arXiv:1301.3781, 2013.
\bibitem[Nash, 1950]{nash1950bargaining}
John F. Nash.
\textit{The Bargaining Problem}.
\textit{Econometrica}, 18(2):155--162, 1950.
\bibitem[Nemhauser et~al., 1978]{nemhauser1978analysis}
George L. Nemhauser, Laurence A. Wolsey, and Marshall L. Fisher.
\textit{An Analysis of Approximations for Maximizing Submodular Set Functions---I}.
\textit{Mathematical Programming}, 14(1):265--294, 1978.
\bibitem[OpenAI, 2024]{openai2024gpt4o}
OpenAI.
\textit{GPT-4o System Card}.
\url{https://openai.com/index/gpt-4o-system-card/}, 2024.
\bibitem[Qwen Team, 2024]{qwen2024report}
Qwen Team, Alibaba Group.
\textit{Qwen2.5 Technical Report}.
arXiv:2412.15115, 2024.
\bibitem[Rinnooy Kan \& Timmer, 1987]{rinnooykan1987stochastic}
A. H. G. Rinnooy Kan and G. T. Timmer.
\textit{Stochastic Global Optimization Methods, Part II: Multi-Level Methods}.
\textit{Mathematical Programming}, 39(1):57--78, 1987.
\bibitem[Robertson et~al., 1994]{robertson1994okapi}
Stephen E. Robertson, Steve Walker, Susan Jones, Micheline M. Hancock-Beaulieu, and Mike
Gatford.
\textit{Okapi at TREC-3}.
In \textit{TREC-3 Proceedings}, pp.\ 109--126, 1994.
\bibitem[Safari \& Babaei, 2026]{safari2025novelty}
Ali Safari and Sahar Babaei.
\textit{Research Novelty in Information Systems Journals After ChatGPT: Differences
Across Institutional Language Contexts}.
arXiv:2603.22510, 2026.
\bibitem[Shen et~al., 2017]{shen2017text}
Tianxiao Shen, Tao Lei, Regina Barzilay, and Tommi Jaakkola.
\textit{Style Transfer from Non-Parallel Text by Cross-Alignment}.
In \textit{Advances in Neural Information Processing Systems}, volume 30, 2017.
\bibitem[Shumailov et~al., 2023]{shumailov2023curse}
Ilia Shumailov, Zakhar Shumaylov, Yiren Zhao, Yarin Gal, Nicolas Papernot, and Ross Anderson.
\textit{The Curse of Recursion: Training on Generated Data Makes Models Forget}.
arXiv:2305.17493, 2023.
\bibitem[Shumailov et~al., 2024]{shumailov2024collapse}
Ilia Shumailov, Zakhar Shumaylov, Yiren Zhao, Nicolas Papernot, Ross Anderson, and Yarin Gal.
\textit{AI Models Collapse When Trained on Recursively Generated Data}.
\textit{Nature}, 631(8022):755--759, 2024.
\bibitem[Socher et~al., 2013]{socher2013sst}
Richard Socher, Alex Perelygin, Jean Wu, Jason Chuang, Christopher D. Manning, Andrew Y.
Ng, and Christopher Potts.
\textit{Recursive Deep Models for Semantic Compositionality Over a Sentiment Treebank}.
In \textit{Proceedings of the 2013 Conference on Empirical Methods in Natural Language
Processing (EMNLP)}, pp.\ 1631--1642, 2013.

\end{thebibliography}
\end{document}